\documentclass[conference]{IEEEtran}
\IEEEoverridecommandlockouts
\usepackage{cite}
\usepackage{amsmath,amssymb,amsfonts,amsthm}
\usepackage{algorithm}
\usepackage{algpseudocode}
\usepackage{graphicx}
\usepackage{textcomp}
\DeclareMathOperator*{\argmax}{arg\,max}

\usepackage{xcolor}
\usepackage{svg}
\usepackage{hyperref}
\usepackage{gensymb}
\newtheorem{theorem}{Theorem}[section]

\newtheorem{lemma}[theorem]{Lemma}

\def\BibTeX{{\rm B\kern-.05em{\sc i\kern-.025em b}\kern-.08em
    T\kern-.1667em\lower.7ex\hbox{E}\kern-.125emX}}

\usepackage[textsize=tiny]{todonotes}

\usepackage[left=0.76in,right=0.76in,top=1in,bottom=0.8in]{geometry} %

\begin{document}

\title{{\fontsize{18}{22}\selectfont Hierarchical Reinforcement Learning and Value Optimization for Challenging Quadruped Locomotion}}

\author{%
  Jeremiah Coholich\textsuperscript{1}, 
  Muhammad Ali Murtaza\textsuperscript{1}, 
  Seth Hutchinson\textsuperscript{1}, and 
  Zsolt Kira\textsuperscript{1}%
  \thanks{\textsuperscript{1}Institute of Robotics and Intelligent Machine, Georgia Institute of Technology, Atlanta, GA, USA. Emails: \{jcoholich, mamurtaza, seth, zkira\}@gatech.edu.}
}

\maketitle

\begin{abstract}
We propose a novel hierarchical reinforcement learning framework for quadruped locomotion over challenging terrain. Our approach incorporates a two-layer hierarchy in which a high-level policy (HLP) selects optimal goals for a low-level policy (LLP). The LLP is trained using an on-policy actor-critic RL algorithm and is given footstep placements as goals. We propose an HLP that does not require any additional training or environment samples and instead operates via an online optimization process over the learned value function of the LLP. We demonstrate the benefits of this framework by comparing it with an end-to-end reinforcement learning (RL) approach. We observe improvements in its ability to achieve higher rewards with fewer collisions across an array of different terrains, including terrains more difficult than any encountered during training. 
\end{abstract}


\section{Introduction}

In recent years, there has been an explosion of interest in using reinforcement learning (RL) for robotic planning and control. It is possible to learn robot legged locomotion policies from scratch in an end-to-end manner \cite{haarnoja2018learning, ha2020learning, rudin2022learning, yu2020learning, agarwal2023legged, yang2021learning}; however, this is typically challenging and requires extensive reward function engineering, hyperparameter tuning, or environment engineering. While RL promises to be a general framework for robots to autonomously acquire a wide variety of skills, legged locomotion poses a difficult learning and control problem due to underactuation and high-dimensional state and action spaces. 

\begin{figure}[t]
\begin{center}
\includegraphics[scale=0.52]{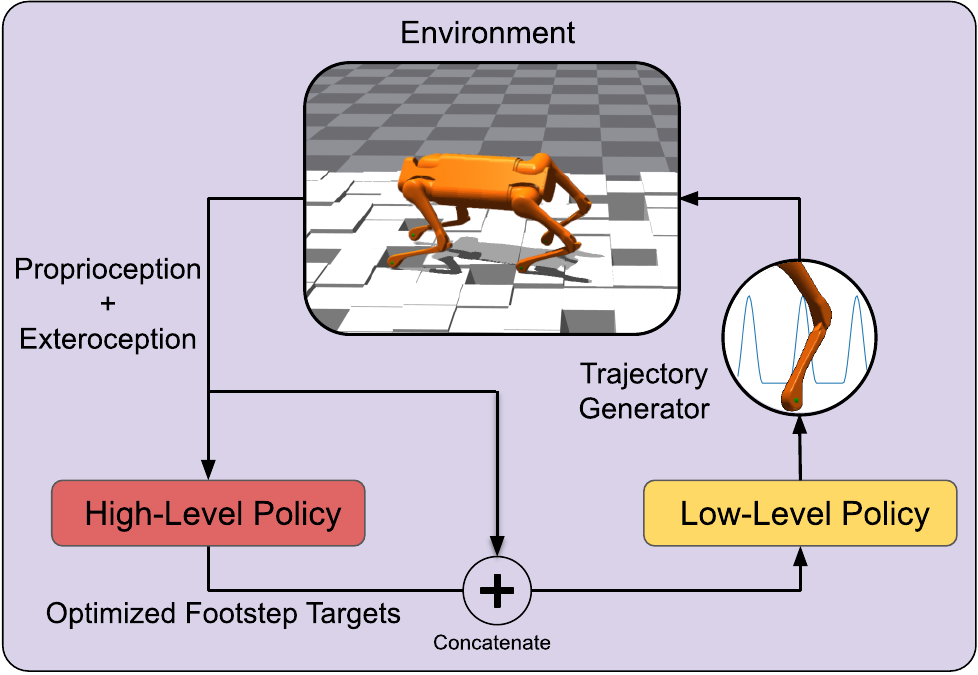}
\caption{The policy architecture incorporating a high-level policy which makes use of the low-level policy's value function for selecting high-value footstep targets}
\end{center}
\end{figure}

To avoid these issues and increase the success rate of learning locomotion policies, researchers began to incorporate various priors into RL algorithms. Most notably, \cite{iscen2018policies} proposes a gait trajectory generator (TG) and limits the RL policy to learning residuals which are added to the TG output. This approach was adopted by others in order to improve development time, sample efficiency, and the success rate of learning locomotion policies \cite{lee2020learning, yu2021visual, escontrela2020zero, iscen2021learning}. We use a similar style of trajectory generator in our proposed approach and in our "end-to-end" reinforcement learning baseline.

Other forms of prior knowledge are feasible as well. Polices can be learned by imitating quadruped animals \cite{peng2020learning}. However, collecting high-quality animal data is difficult to scale because it requires bringing animals into a motion capture lab. Additionally, there exists a significant morphology gap between quadruped animals and state-of-the-art quadruped robots, which do not have ankles for example. Other approaches combine a residual RL-policy with online trajectory optimization and whole-body control \cite{gangapurwala2021real, gangapurwala2020guided}. DeepGait constrains learning with model-based feasibility criteria, bypassing the physics simulator \cite{tsounis2020deepgait}. In contrast to all of these approaches, our proposed method only relies on an open-loop TG and otherwise learns locomotion directly from interactions with the physics simulator.

Hierarchical RL methods have also been developed for other legged embodiments. In ALLSTEPS, the authors train bipedal robots in simulation to walk on increasingly difficult stepping-stone sequences \cite{xie2020allsteps}. However, their foot placement sequence is fixed and cannot be optimized when many stepping stones are present. Li \textit{et al.} proposed a hierarchical RL method for hexapods in which a high-level policy performs MPC-style rollouts over a set of learned low-level primitives \cite{li2020learning}. This requires learning dynamics models for each primitive. In our approach, the high-level policy requires no extra learning once the low-level policy is trained.

In this work, we learn policies that find optimal foot placements on terrain with gaps and height variation. In this domain, legged robots clearly trump wheeled robots, since legged robots require only small, discrete contacts with terrain. Planning these contacts, or footstep placements, is therefore crucial in unlocking the full capability of legged robots. We posit that biasing our policy architecture to focus on footstep placements will improve our ability to traverse such challenging terrain. In addition, the use of a hierarchical framework provides a modular structure, which accommodates the swapping of components.

Our method involves a two-layer hierarchy, where footstep target locations are passed from the high-level policy (HLP) to the low-level policy (LLP). In this setup, we first train the LLP to control a simulated quadruped robot to hit a sequence of randomly generated footstep targets. The HLP then finds optimal footstep locations by leveraging the value function obtained during the training of the LLP. Other works contain similar online optimization approaches. QT-Opt is a technique for online optimization over a learned Q-function using the derivative free cross-entropy method \cite{kalashnikov2018scalable}. Our work includes the addition of a hierarchy and an optimization term which makes our architecture more flexible. We use a combination of derivative-free and derivative-based optimization methods. \cite{li2022hierarchical} uses a similar hierarchical approach leveraging a low-level policy's value function, but focuses on the offline RL setting where distribution shift from the offline training data is a significant concern.

We can summarize the main contributions of this paper as follows:
\begin{itemize}
    \item A hierarchical learning-based quadruped control architecture in which the high-level footstep policy is obtained without requiring additional training. 
    \item An online value-optimization process for selecting low-level policy goals, obtained without additional environment samples beyond low-level policy training.
    \item Validation of the capability of the proposed methodology to generalize beyond its training environment, compared with an end-to-end RL policy on the task of quadruped locomotion over rough terrain
\end{itemize}

The rest of the paper is organized as follows: Section \ref{sec:LLP_Training} gives RL preliminaries and discusses LLP training. Section \ref{sec:HLP} outlines the HLP and its associated action space and objective function. Experiments and results are presented in Section \ref{sec:results}, and future work and conclusions are given in section \ref{sec:conclusion}. Video results and code are available at: \href{http://www.jeremiahcoholich.com/publication/hrl_optim/}{www.jeremiahcoholich.com/publication/hrl\_optim/}.

\section{Low-Level Policy Training}\label{sec:LLP_Training}

The low-level policy (LLP) is goal-conditioned and  outputs actions directly to the robot. Our method leverages the LLP's value function, which gives the expected cumulative reward for a given state, goal, and policy. As a result, an actor-critic RL algorithm must be used. In theory, the algorithm can either be on-policy or off-policy.

We train the LLP to hit a randomized sequence of procedurally generated footstep targets. Both the actor and critic networks take the same input consisting of goal footstep target locations and robot observations.

\subsection{RL Preliminaries}
We formulate the task of hitting footstep targets as a partially-observable Markov decision process, which is a tuple $ \left(S, \mathcal{O}, A, p, r,\rho_{0}, \gamma\right) $. Here, $S$ is the set of environment states, $\mathcal{O}$ is the set of observations, $A$ is the set of policy actions, $p: S \times A \rightarrow S$ is the transition function of the environment, $r: S \times A \times S \rightarrow \mathbb{R}$ is the reward function, $\rho_{0}$ is the distribution of initial states, and $\gamma$ is a discount factor. Our goal is to find an optimal policy \(\pi^*: S \rightarrow A\) that maximizes the discounted sum of future rewards \(J(\pi)\) over time horizon $H$.

\begin{equation}
J(\pi)=\mathbb{E}_{\{s_i, a_i\}_0^{H} \sim \pi, \, \rho_0} \Bigg[ \sum_{t=0}^{H} \gamma^{t} r\left(\mathbf{s}_{t}, \mathbf{a}_{t}, \mathbf{s}_{t+1}\right) \Bigg] 
\end{equation}
\begin{equation}
\pi^{*}=\argmax_{\pi} J(\pi)
\end{equation}

We use the proximal policy optimization (PPO) \cite{schulman2017proximal}, an on-policy actor-critic method, to solve for $\pi^*$ with $\lambda = 0.99$.  Additionally, we train a value network to predict the value of a state given the current policy. The value network is trained with the mean-squared error loss and generalized advantage estimation (GAE) \cite{schulman2015high} to stabilize training.

\begin{equation}
\begin{aligned}
V_{\pi}\left(\mathbf{s_0}\right) := \mathbb{E}_{\pi,\, p} \Bigg[ 
\sum_{t=0}^{H} \gamma^{t} r\left(\mathbf{s}_{t}, \mathbf{a}_{t}, \mathbf{s}_{t+1}\right) \Bigg] \\
\end{aligned}
\end{equation}

\begin{equation*}
\mathbf{a}_t \sim \pi(\cdot\mid \mathbf{s}_t), \ 
\mathbf{s}_{t+1} \sim p(\cdot\mid \mathbf{s}_t,\mathbf{a}_t )
\end{equation*}

The policy and value networks are parameterized as separate multilayer perceptrons with two hidden layers of size 128.

\subsection{Action Space}

We use the Policies Modulating Trajectory Generators (PMTG) architecture \cite{iscen2018policies} with the foot trajectories given in \cite{lee2020learning}. Our 15-dimensional action space consists of trajectory generator frequency, step length, standing height, and 12 residuals corresponding to the 3D position of each foot. The trajectory generator outputs foot positions in the hip-centered frame (as defined in \cite{lee2020learning}). These are converted into joint positions with analytical inverse kinematics and tracked with PD control. The trajectory generator cycle is synced to a phase variable $\phi_t \in [0, 2\pi)$, where $\mathcal{S}:= [0.25\pi, 0.75\pi]\cup[1.25\pi, 1.75\pi]$ represents the swing phase of each leg and $[0, 2\pi)/\mathcal{S}$ is the support phase.

We design the TG to output a trotting gait, which means two feet have targets at any given time. $\mathcal{N}$ are the pair of feet that the robot has active targets at time $t$ with $\mathcal{N} \in \{\{1,\,4\},\, \{2,\,3\}\}$. The foot indices in numerical order correspond to the front-left, front-right, rear-left, and rear-right feet. 

\subsection{Observation Space}\label{subsec:observation_space}
The policy observation is a vector $\mathcal{O}_t = \{\mathbf{x}$, $ \mathbf{\dot{x}}$, $ \tau,\,\mathbf{O}$, $ \mathbf{c}$, $ \mathbf{p}$, $ \mathbf{f}$, $\cos{\phi}$, $\sin{\phi}$, $ \mathbf{a_{t-1}}$, $ \mathbf{a_{t-2}}$, $ \mathcal{F}\}$ where $\mathbf{x} \in \mathbb{R}^{12}$ represents the foot positions in the hip-frame, $\tau \in \mathbb{R}^{12}$ is the joint torques, $\mathbf{O} \in \mathbb{R}^{4}$ is the IMU data consisting of $\{\theta_{\text{roll}}, \theta_{\text{pitch}}, \dot{\theta}_{\text{roll}}, \dot{\theta}_\text{{pitch}}\}$, $\mathbf{c} \in \{0, 1\} \subset \mathbb{R}^4$ is a vector giving the contact state of each foot, $\mathbf{p}  = \{p_{1,x}, p_{1,y}, p_{2,x}, p_{2,y}, p_{3,x}, p_{3,y}, p_{4,x}, p_{4,y}\} \in \mathbb{R}^8$ gives the x and y distances from each foot to the next (for $i \in \mathcal{N}$) or previous ($i \notin \mathcal{N}$) footstep targets, $\mathbf{f} \in \{0,1\} \subset \mathbb{R}^4$ is a multi-hot encoding of $\mathcal{N}$, $\phi \in \mathbb{R}$ is the phase of the trajectory generator, $\mathbf{a_{t-1}}$ and $\mathbf{a_{t-2}}$ give the previous two actions taken by the policy, and $\mathcal{F}$ is a scan of points around each foot.

\subsection{Reward Function}\label{subsection:training_env}
\label{section:training_environment}
The reward function for the LLP encourages hitting footstep targets and contains additional terms to encourage a reasonable gait. The reward function terms are as follows:

\subsubsection{Footstep Target Reward}\label{subsubsec:hit_footstep_target} 
Equation \ref{eq:hit_targets_rew} defines this reward term, where $h_{i,t} \in \{0, 1\}$ indicates whether or not foot $i$ has hit its footstep target at time $t$. A target is considered hit if the foot makes contact with at least 5 N of force in a 7.5 cm radius around the target while the trajectory generator is in the contact phase for that foot. We define $\mathbf{d}_{i,t}$ as the distance in the xy plane from the foot center to the target center. If the robot hits both active footstep targets at once, the reward for each foot is added, the total is tripled, and the environment advances to the next pair of targets. This reward function is inspired from \cite{xie2020allsteps} and is given by
\begin{equation}
\label{eq:hit_targets_rew}
    \kappa_{FT}\ \left[2\ {\prod_{i \in \mathcal{N}}} h_{i,t} + 1 \right]{{\sum_{i \in \mathcal{N}}}\ h_{i, t} \left[1 + 0.5 \left(1 - \frac{\mathbf{d}_{i, t}}{ \mathbf{d}_{\text{hit}}}\right)\right]}
\end{equation}

where $\kappa_{FT}$ is the weighting term for the reward function and $\mathbf{d}_{\text{hit}}$ is the xy distance threshold for hitting a footstep target (set to 7.5 cm). The factor $ \left(2.0\ {\prod_{i \in \mathcal{N}}} h_{i,t} + 1 \right)$ triples the per-foot rewards if both footstep targets are achieved on the same timestep.

\subsubsection{Velocity Towards Target}
To provide denser rewards that encourage hitting footstep targets, we reward foot velocity towards targets.  
\vspace{-5pt}
\begin{equation}
    \kappa_{VT} {{\sum_{i \in \mathcal{N}}}} \dot{\mathbf{d}}_{i,t}
\end{equation}

\subsubsection{Smoothness Reward}
To ensure a smooth robot motion, we added a penalizing term to the norm of second-order finite differences of the actions, where $\mathbf{a}_t$ is the action at timestep $t$.
\vspace{-5pt}
\begin{equation}
    \kappa_{S}\|\mathbf{a_{t}} - 2\mathbf{a_{t-1}} + \mathbf{a_{t-2}} \|_2
\end{equation}

\subsubsection{Foot Slip Penalty}
This term penalizes xy translation greater than 2 cm for feet that are in contact. $c_{i,t}$ gives the vertical contact force for foot $i$ at time $t$ in Newtons. $x_{i, t} \in \mathbb{R}^2$ is the position in meters on the x-y plane for foot $i$ at time $t$. 
\vspace{-5pt}
\begin{equation}
-\kappa_{SL}\left|\left\{ i \,\middle|\, 
\substack{1\le i \le 4 \\[1mm]
\|x_{i,t} - x_{i,t-1}\|_2 > 0.02 \\[1mm]
c_{i,t} > 0 \\[1mm]
c_{i,t-1} > 0
}
\right\}\right|
\end{equation}

\subsubsection{Foot Stay Reward}
A trotting gait requires two feet to have active footstep targets at any time. To prevent the robot from immediately moving its feet off of footstep targets after they are hit, we reward the agent for keeping its feet on previous targets.
\begin{equation}
    \kappa_{FS}{{\sum_{i \in \{1,2,3,4\}\backslash\mathcal{N}}}\ h_{i, t}   \left[  1 +  \frac{1}{2}\left(1 - \frac{\mathbf{d}_{i, t}}{ \mathbf{d}_{\text{hit}}}\right)\right]}
\end{equation}

\subsubsection{Collision Penalty}
We add a penalty if any robot linkage collides with another linkage or with terrain, excluding the case of robot feet colliding with terrain. The penalty is given by Equation \ref{eq:collision}, where $\mathbf{g}$ is the number collisions.
\begin{equation}
\label{eq:collision}
    -\kappa_{C}\mathbf{g}
\end{equation}

\subsubsection{Trajectory Generator Swing Phase Reward}
This term rewards the trajectory generator for entering the swing phase $\phi_t$, weighted by the frequency of the trajectory generator ($f_{\text{PMTG}}$). This term prevents the RL algorithm from learning a degenerate policy that remains at the same place and collects maximum rewards for foot stay, foot slip, and smoothness. This reward is given by Equation \ref{eq:swing}, where ${\mathbf{1}}\{\cdot\}$ is the indicator function.
\begin{equation}
\label{eq:swing}
{\mathbf{1}}_{\mathcal{S}}\{\phi_t\}\ f_{\text{PMTG}}    
\end{equation}

\begin{figure}[]
\begin{center}
\includegraphics[scale=0.5]{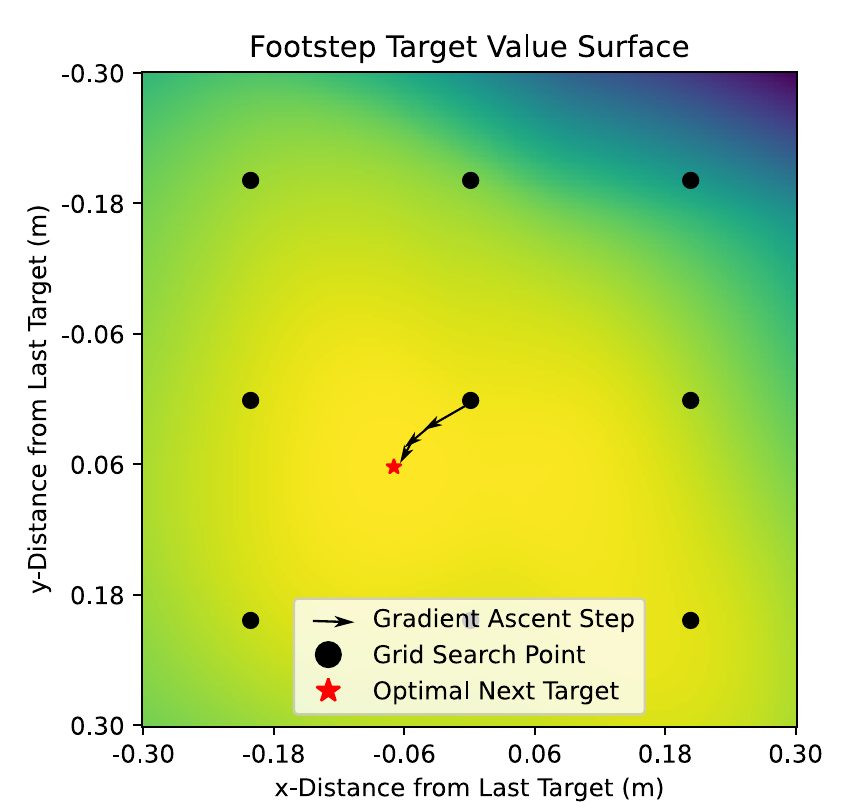}
\caption{A visualization of the high-level policy optimization approach on a 2D slice of the value function goal-space}
\label{fig:optimization}
\end{center}
\end{figure}

\section{High-Level Policy}
\label{sec:HLP}
The purpose of the HLP is to choose a goal, or footsteps target, for the LLP. No additional samples from the environment or neural network parameter updates are required for the HLP once the LLP is fully trained. The HLP makes use of the LLP value function, which is typically discarded after RL training.  

First, we describe the action space of the HLP and its objective function. Then, we discuss the online optimization process used to find optimal actions for the LLP.
\subsection{Action Space} 
The HLP action space is a continuous eight-dimensional space that encodes the x and y relative positions of the next footstep targets for all four feet of the quadruped, which is the vector $\textbf{p}$ defined in Section \ref{subsec:observation_space}.

$$A_{HLP} := \textbf{p} \subset \mathcal{O}$$
In addition to the current observation $\mathbf{o}_t$, the HLP also receives the robot's yaw angle, $\theta_{yaw}$. This necessary to define a direction for travel.

\subsection{Objective Function}

\begin{figure*}[t]
\begin{center}
\includegraphics[scale=0.53]{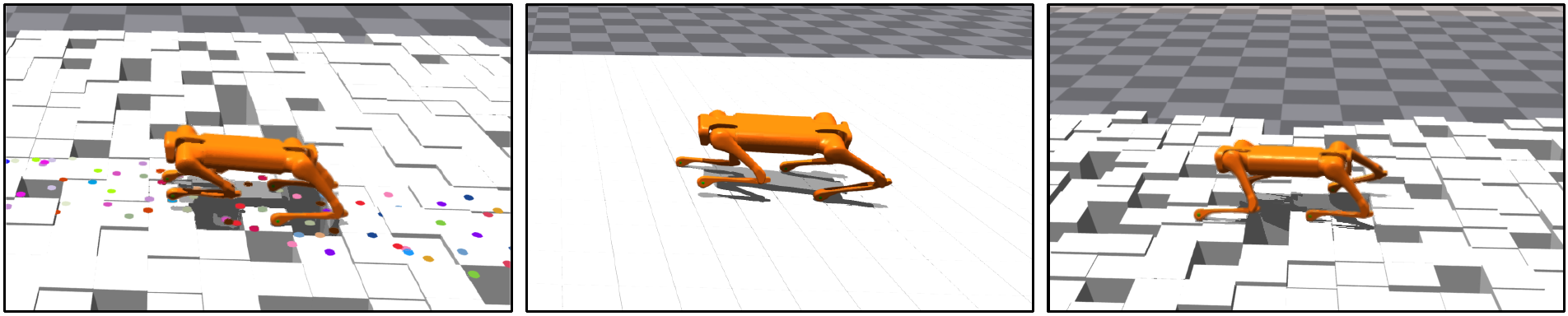}
\end{center}
  \caption{\textbf{Left} The training environment with procedurally generated footstep targets \textbf{Center}: The least-challenging test environment, with 100\% infill and no height variation. \textbf{Right}: The most-challenging test environment with 80\% infill and 10 cm height variation}
\label{fig:rough_terrain_steps}
\end{figure*}

The objective function of the HLP includes the expected discounted rewards of the LLP, which is estimated by the LLP value function, plus an auxiliary objective $\mathbf{H}$. The auxiliary objective is necessary since simply choosing the highest-value footstep targets will yield solutions where the robot steps in place. $\mathbf{H}$ is designed to encourage locomotion in a particular direction and is parameterized by a heading angle $\alpha$ and a weight $\kappa_{HD}$. The robot yaw $\theta_{yaw}$ is used to map the targets in robot frame to the world frame.

\begin{equation}
\mathbf{H} = \begin{bmatrix}
\cos{\alpha} & \sin{\alpha}\\
\end{bmatrix}\
R_z(\theta_{yaw})
\begin{bmatrix}
p_{a,x} & p_{b,x} \\
p_{a,y} & p_{b,y} \\
\end{bmatrix}\
\begin{bmatrix}
1\\
1\\
\end{bmatrix}\
\label{eq:directional_term}
\end{equation}
The directional term is given by Equation \ref{eq:directional_term} where $R_z(\theta_{yaw})$ is a 2D rotation matrix. The objective function for the HLP at time $t$ is given by Equation \ref{eq:full_hlp_objective}. We would like to solve the optimization problem given in Equation \ref{eq:hlp_optimization}.
\begin{equation}
\label{eq:full_hlp_objective}
R_{\text{HLP}} := V(\mathbf{s}_t) + \kappa_{HD} \mathbf{H}
\end{equation}

\begin{equation}
\label{eq:hlp_optimization}
\mathbf{p}^* = \argmax _{\mathbf{p}}\;R_{\text{HLP}}
\end{equation}

The hyperparameter $\kappa_{HD}$ controls the tradeoff between picking targets that maximize the expected success of the LLP with picking targets that advance the robot's movement in direction $\alpha$.

\subsection{Optimization}
\begin{algorithm}
\caption{Grid Search Initialized Gradient Ascent for HLP Optimization}
\begin{algorithmic}[1]
\State \textbf{Input:} Low-level policy value function $V(s_t)$, directional objective $H$, grid search bounds $B$, grid resolution $R$, learning rate $\eta$, number of gradient ascent iterations $N$
\State \textbf{Output:} Optimal footstep targets $p^*$

\State $p_{\text{best}} \gets \mathbf{0}$  \Comment{Initialize the best footstep target}
\State $R_{\text{best}} \gets -\infty$ \Comment{Initialize best reward}
\State $P_{\text{grid}} \gets \text{GenerateGrid}(B, R)$ \Comment{Generate grid points within bounds}

\State \Comment{Grid Search Step}
\For{$p \in P_{\text{grid}}$}
    \State $R \gets V(s_t) + \kappa_{\text{HD}} H(p)$ \Comment{Evaluate HLP objective for each $p$}
    \If{$R > R_{\text{best}}$}
        \State $p_{\text{best}} \gets p$
        \State $R_{\text{best}} \gets R$
    \EndIf
\EndFor

\State \Comment{Gradient Ascent Step}
\State $p \gets p_{\text{best}}$ \Comment{Initialize $p$ with the best grid search result}
\For{$i = 1$ to $N$}
    \State $\nabla R(p) \gets \nabla_p [V(s_t) + \kappa_{\text{HD}} H(p)]$ \Comment{Compute gradient of HLP objective}
    \State $p \gets p + \eta \cdot \nabla R(p)$ \Comment{Update footstep targets using gradient ascent}
\EndFor

\State \Return $p$
\end{algorithmic}
\end{algorithm}

There are multiple options for solving equation \ref{eq:hlp_optimization}, including gradient-based optimization methods, since both the value function and $\mathbf{H}$ are differentiable with respect to $\textbf{P}$.
Leveraging the low-dimensionality of $\textbf{P}$, we  solve Equation \ref{eq:hlp_optimization} with grid-search initialized gradient ascent, the approach shown in Figure \ref{fig:optimization}. We first discretize the 8-dimensional space of $\mathbf{d}_{next}$ into a box $[-B, B]_8$ with $R$ points per axis and query the objective function at each point. The optimum point of the grid search is used as the initialization for gradient ascent. The full algorithm is given in Algorithm 1. In our experiments, we set $\eta = 10^{-4}$, $R = 5$, $B=15\text{cm}$, and $N=5$. These values were chosen to minimize runtime compute requirements without reducing the search space too much and sacrificing performance. We set $\alpha$ from Equation \ref{eq:directional_term} to zero degrees, which corresponds to forward motion.

We present a lemma to lower-bound the error of grid search initialized gradient ascent.

\begin{lemma}
The expected initial error for grid search initialized gradient ascent is smaller than or equal to the expected initial error for random initialized gradient ascent i.e
\begin{equation}
f(x^*) - f(x_{\text{best}}) \leq \mathbb{E}_{x_0 \sim U(\cdot)} [f(x^{*}) - f(x_0)] 
\end{equation}
\begin{equation*}
    x_{\text{best}} = \argmax_{x\in \textit{G}} f(x)
\end{equation*}
where $f(\cdot)$ is the objective function, $\mathbb{E}_{x_0 \sim U(\cdot)}$ denotes the expectation over a uniformly random initialization over the support of $x$, $\textit{G}$ is set of points for grid search, and $x^{*}$ is the parameter which yields the global maximum.
\end{lemma}
\begin{proof}
The proof follows from the observation that $\mathbb{E}_{x_0\in \textit\{G\}} f(x_{\textit{best}}) \geq \mathbb{E}_{x_0 \sim U(\cdot)}f(x_0)$ and the rest of the proof is trivial. 
\end{proof}

\section{Experiments and Results} \label{sec:results}
We train our LLP in simulation using NVIDIA Isaac Gym \cite{makoviychuk2021isaac}. We sample 100 steps from 4,000 environments for a total of 400,000 samples per policy update. Each policy is trained for 750 iterations giving 300 million total samples. The training environment consists of terrain with 90\% infill and terrain blocks with heights varying by up to 5 cm. In addition to our proposed method, we also train an end-to-end reinforcement learning policy for comparision. The experiments in this section are designed to answer the following questions:
\begin{itemize}
    \item How does our proposed optimization method perform in quadruped locomotion over challenging terrain compared to a PMTG \cite{iscen2018policies} end-to-end policy? 
    \item Does our proposed approach enable higher LLP rewards than achieved during training?
\end{itemize}

In this section, we will first give more details on LLP training, then define the end-to-end RL policy, and finally discuss results on various test terrains.

\subsection{Training Environment}
\label{appendix:training_env_details}
\begin{figure}[]
    \centering
    \includegraphics[scale=0.7]{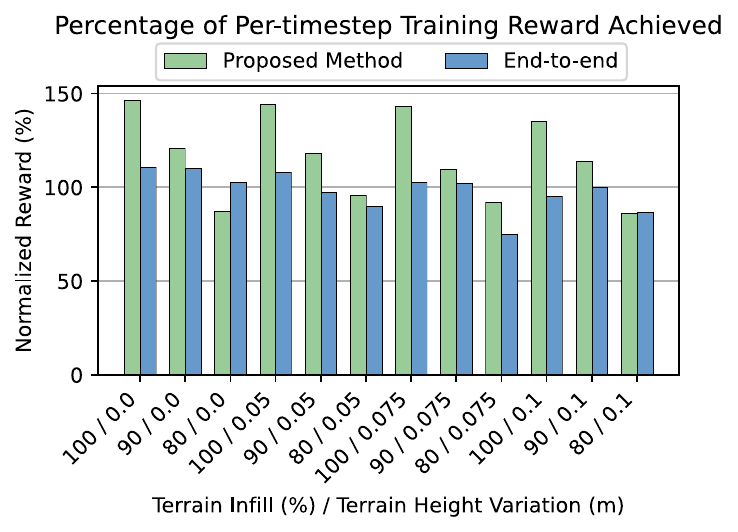}
    \caption{A comparison of the proposed value-function-based approach with an end-to-end RL policy. Each bar represents the average result of five rollouts. The HLP enables the LLP to obtain higher normalized rewards than the end-to-end policy in 10/12 terrains. Even on terrains much more difficult than the ones encountered in training (90 / 0.05), our method achieves normalized rewards greater than 100\%.}
    \label{fig:rough_terrain_small_plot}
\end{figure}

We generate sequences of footstep targets corresponding to a trotting gait, where the robot is tasked with hitting targets for two feet at a time, alternating between the front-left and rear-right feet and the front-right and rear-left feet. We have found empirically that the trotting gait is the most suitable for implementation on the Aliengo robot in terms of robustness and speed. Each environment contains a sequence of footstep targets parameterized by a random step length sampled from $U(0, 0.2)$ m and a random heading sampled from $U(0^{\degree}, 360^{\degree})$. Additionally, all targets are independently randomly shifted by $U(-0.1, 0.1)$ m in the x and y directions. The training terrain is pictured in Figure \ref{fig:rough_terrain_steps}.

\subsection{End-to-End RL Policy}
\looseness=-1 We train an RL policy on the same training terrain with the same trajectory generator action space  \cite{iscen2018policies} using PPO. The reward function for the end-to-end policy contains all of the reward terms and coefficients in section \ref{subsection:training_env} sans the footstep target reward and the velocity towards target reward. Additionally, to encourage forward locomotion, we add the reward term given by Equation \ref{eq:ee_fwd}, where $\mathbf{v}$ is the robot velocity and $\kappa_{VX}$ is set to 1.0. The robot velocity (m/s) is clipped to encourage the development of stable gaits for a fair comparison.

\begin{equation}
\label{eq:ee_fwd}
    \kappa_{VX}\cdot\text{clip}(\textbf{v}_x, -\infty, 0.5)
\end{equation}

Additionally, we add a term to penalize velocity in the y-direction, given below in Equation \ref{eq:ee_ypen}.
\begin{equation}
\label{eq:ee_ypen}
    -\kappa_{VY}|\textbf{v}_y|
\end{equation}

\subsection{Locomotion on challenging terrain}

\begin{figure}[t]
    \centering
    \includegraphics[scale=0.7]{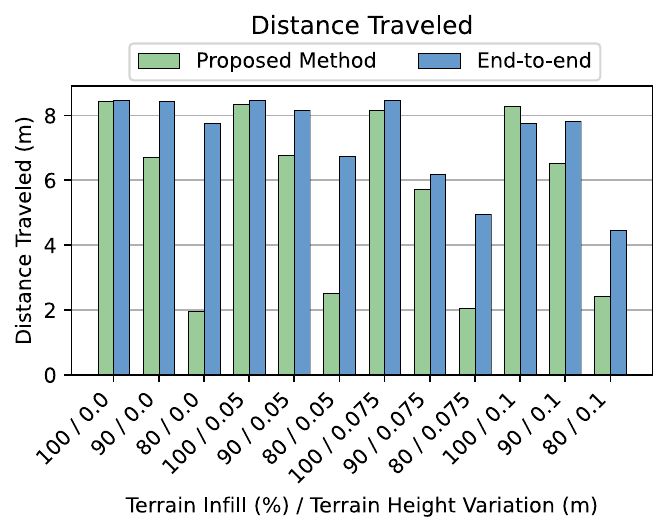}
    \caption{Distance traveled in meters for each approach across different test terrains. Each bar represents the average result of five rollouts.}
    \label{fig:distance_traveled_plot}
\end{figure}

We test the trained policies on environments of varying difficulty, depicted in Figure \ref{fig:rough_terrain_steps}. Our simulation terrain varies in difficulty along two axes: infill and height variation. An infill lower than 100\% indicates gaps or holes in the terrain. The heights of terrain blocks are uniformly randomized such that the maximum range of heights is equal to a terrain height variation parameter. We run experiments on terrains with 100, 90, and 80 percent infill and 0, 5, 7.5, and 10 cm height variation. For all experiments with the proposed method, we set $\alpha$ in Equation \ref{eq:directional_term} to 0.0, which corresponds to rewarding footstep targets set in the positive x-direction. The weight of the directional term ($\kappa_{HD}$ in Equation \ref{eq:hlp_optimization}) is set to 50.0 for all experiments.

\subsubsection{Percentage of Per-Timestep Training Reward Achieved}
We use reward as a proxy for overall performance of each method, since it encodes the core objective of hitting footstep targets (or forward velocity, for the end-to-end policy) in additional to other practical concerns such as avoiding collisions and slipping. Since the proposed method and the end-to-end method have differing reward functions, we normalize rewards by the maximum reward achieved during training. Figure \ref{fig:rough_terrain_small_plot} plots the normalized rewards achieved on our array of test terrains. The HLP optimization process enables higher rewards than those achieved in training in eight out of 12 terrains. The end-to-end policy cannot benefit from online optimization, giving a lower normalized reward than our proposed method on 10 out of 12 terrains.
\subsubsection{Distance Traveled}
Figure \ref{fig:distance_traveled_plot} shows that our proposed method travels a shorter distance than the end-to-end method in all but two environments. We posit that this is due to our objective of picking high-value, or "safe", footstep targets to execute. The largest gaps in distance occur in the 80\% infill environments, where the presence of holes stops forward progress, since it is impossible to hit a footstep target over a hole. Our method's conservativism in such a scenario is highlighed in the next subsection.
\subsubsection{Collisions}
Figure \ref{fig:collisions} gives the average number of collisions per timestep. In two environments with 80\% and 90\% infill, the end-to-end policy experiences an extremely high number of collisions, nearing an average of one collision per timestep. We believe this is due to the end-to-end policy getting stuck in a terrain hole, which does not occur with our proposed method.

\section{Conclusion} \label{sec:conclusion}
We propose a hierarchical reinforcement learning framework that improves performance on simulated quadruped locomotion over difficult terrain as demonstrated through higher normalized rewards and lower numbers of collisions. By leveraging a novel approach where the HLP optimizes over footstep targets using the LLP value function, we remove the requirement for additional environment samples or neural network parameter updates beyond LLP training.
\begin{figure}[t]
    \centering
    \includegraphics[scale=0.7]{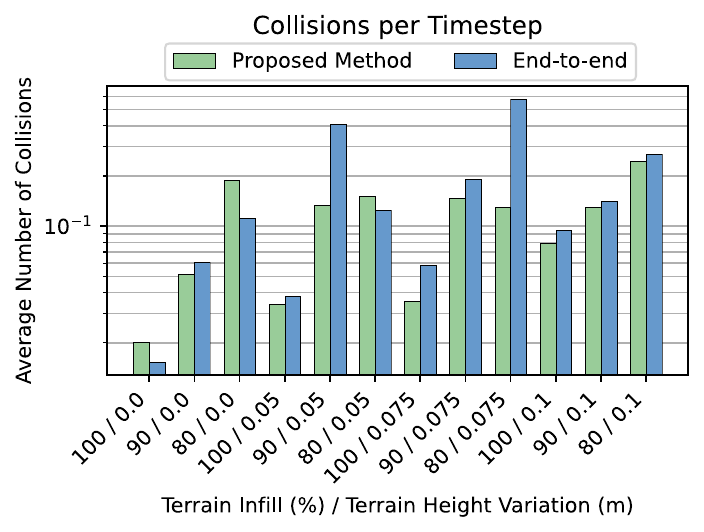}
    \caption{A collision is defined as a robot linkage (excluding feet) in contact with terrain, or a robot linkage in contact with another linkage. Multiple collisions may be counted at a single timestep. Each bar represents the average result of five rollouts.}
    \label{fig:collisions}
\end{figure}
Future work will focus on conducting hardware experiments to further validate the applicability of our approach in physical environments. Additionally, we aim to explore integrating model-based controllers as the low-level policy, as modularity is a practical benefit of hierarchical reinforcement learning. A combination of model-based and learning-based approaches offers a promising direction for further improving the adaptability and reliability of quadruped locomotion in complex real-world applications.

\bibliographystyle{IEEEtran}
\bibliography{references}

\end{document}